\documentclass[twoside,leqno,twocolumn]{article}

\usepackage[letterpaper]{geometry}

\usepackage{siamproceedings}

\usepackage[T1]{fontenc}
\usepackage{amsfonts}
\usepackage{graphicx}
\usepackage{epstopdf}
\usepackage{enumitem}
\usepackage{algorithmic}
\usepackage{xurl}
\usepackage{arydshln}
\usepackage{multirow}
\usepackage{booktabs}
\usepackage{caption}
\usepackage{subcaption}
\usepackage{dblfloatfix}
\usepackage{xcolor}

\definecolor{visioBlue}{RGB}{0,112,192}
\definecolor{visioGreen}{RGB}{0,176,80}
\definecolor{visioRed}{RGB}{255,0,0}
\definecolor{visioYellow}{RGB}{255,192,0}

\ifpdf
  \DeclareGraphicsExtensions{.eps,.pdf,.png,.jpg}
\else
  \DeclareGraphicsExtensions{.eps}
\fi

\newsiamremark{remark}{Remark}
\newsiamremark{hypothesis}{Hypothesis}
\crefname{hypothesis}{Hypothesis}{Hypotheses}
\newsiamthm{claim}{Claim}

\usepackage{amsopn}

\begin{document}

\newcommand\relatedversion{}

\title{\Large Unpacking Hateful Memes: Presupposed Context and False Claims\relatedversion}


\author{
Weibin Cai\thanks{Data Lab, Department of EECS, Syracuse University, New York, USA. Emails: \email{weibin44@data.syr.edu}, \email{reza@data.syr.edu}.}
\and Jiayu Li\thanks{Amazon Web Services. Email: \email{jlijiayu@amazon.com}}
\and Reza Zafarani\footnotemark[1]
}

\date{}

\maketitle
\newcommand{\pmval}[1]{{\scriptsize{$\pm #1$}}}


\fancyfoot[R]{\scriptsize{Copyright \textcopyright\ 2026 by SIAM\\
Unauthorized reproduction of this article is prohibited}}





\begin{abstract} While memes are often humorous, they are frequently used to disseminate hate, causing serious harm to individuals and society. Current approaches to hateful meme detection mainly rely on pre-trained language models. However, less focus has been dedicated to \textit{what makes a meme hateful}. Drawing on insights from philosophy and psychology, we argue that hateful memes are characterized by two essential features: a \textbf{presupposed context} and the expression of \textbf{false claims}. To capture presupposed context, we develop \textbf{PCM} for modeling contextual information across modalities. To detect false claims, we introduce the \textbf{FACT} module, which integrates external knowledge and harnesses cross-modal reference graphs. By combining PCM and FACT, we introduce \textbf{\textsf{SHIELD}}, a hateful meme detection framework designed to capture the fundamental nature of hate. Extensive experiments show that SHIELD outperforms state-of-the-art methods across datasets and metrics, while demonstrating versatility on other tasks, such as fake news detection.\\  

\vspace{-2mm}

\noindent \textcolor{red}{\textbf{Disclaimer:} \textit{This paper contains discriminatory content that may be disturbing to some readers.}} 

\end{abstract}

\section{Introduction.}
\label{sec:introduction}
\textit{Meme}, a term introduced by Richard Dawkins in 1976 from the Greek \textit{mimeme} (``imitation''), refers to units of cultural information, such as ideas and behaviors. With the rise of the internet, ``meme'' now often refers to humorous or satirical image-text combinations that convey individual ideologies and rapidly evolve online. 

However, in some cases, the rapid spread of memes has been exploited to disseminate hate, reinforcing societal biases and threatening social harmony~\cite{pandiani2024toxic}. At the individual level, exposure to hate speech has been shown to harm mental health~\cite{wachs2022online,saha2019prevalence,pluta2023exposure,impacthate}. At the societal level, hate speech can deepen social divisions~\cite{italian2025hate}, and studies have shown its association with real-world criminal behavior~\cite{hatecrime,hateoffline,lewis2019online}. To mitigate these risks, internet companies are legally required to comply with national regulations~\cite{uk2023law,france2020law,germany2017law,aus20law,eu25}; for instance, France mandates the removal of hateful content within 24 hours or imposes fines of up to €1.25 million~\cite{france2020law}. Figure~\ref{fig:example} shows an example of a hateful meme, implicitly reflecting the stereotype that Black people are inherently more violent and criminal, and the resulting hate toward the Black community.


\begin{figure}[t]
  \centering
  \includegraphics[scale=0.225]{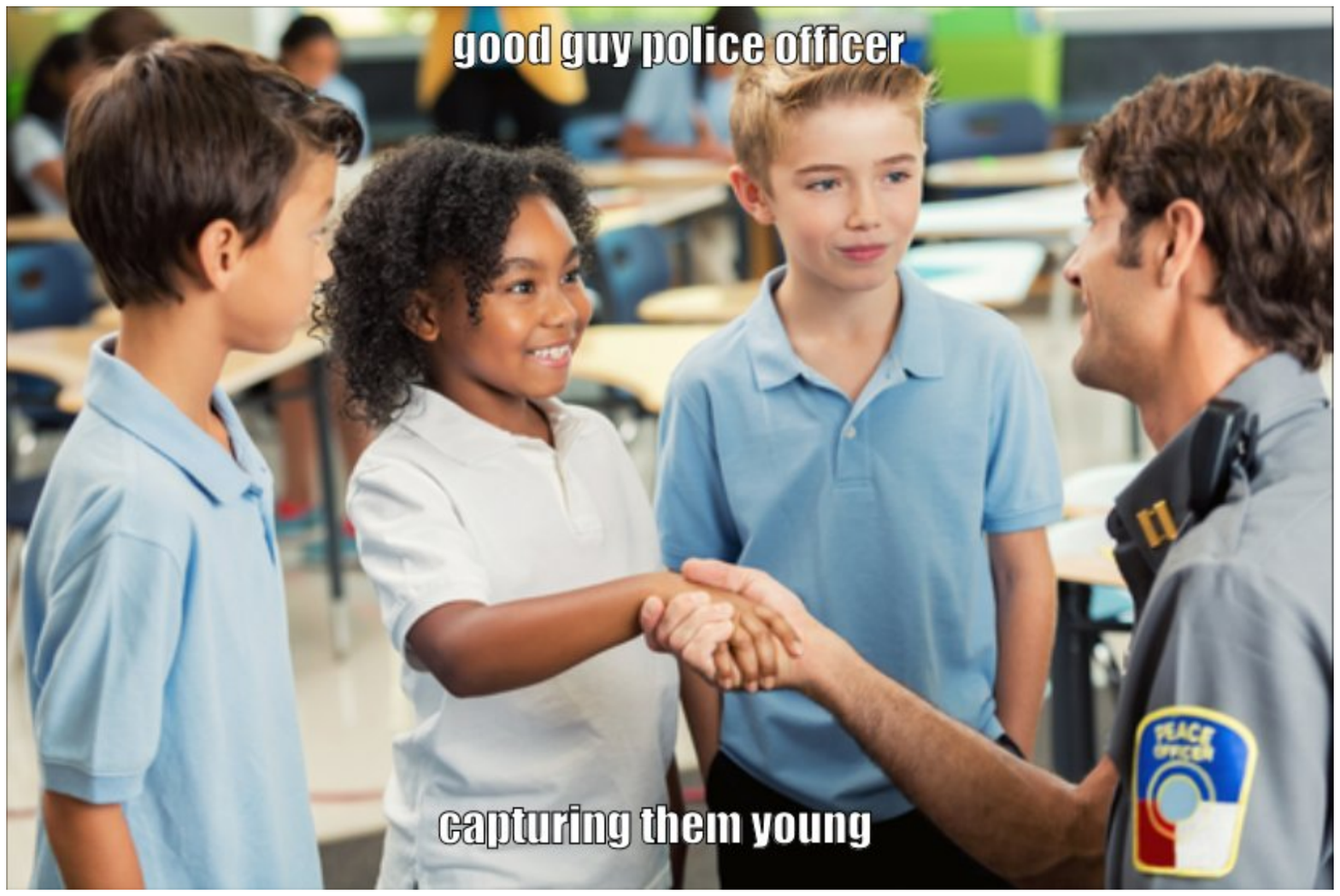} 
  \vspace{-4mm}
  
  \caption{\small An example of a hateful meme. Meme text: \textit{good guy police officer, capturing them young}.}
  \vspace{-5mm}
  
  \label{fig:example}
\end{figure}

\begin{figure*}[!th]
    \centering
    \begin{subfigure}{0.32\textwidth}
        \includegraphics[width=\linewidth]{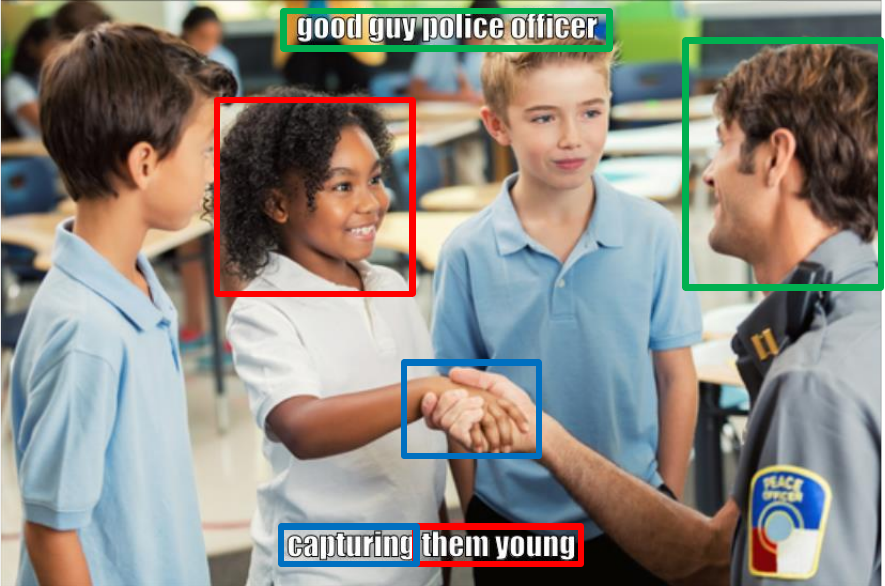}
        \caption{Presupposed context $+$ false claims}
        \label{fig:example1}
    \end{subfigure}
    \hfill
    \begin{subfigure}{0.32\textwidth}
        \includegraphics[width=\linewidth]{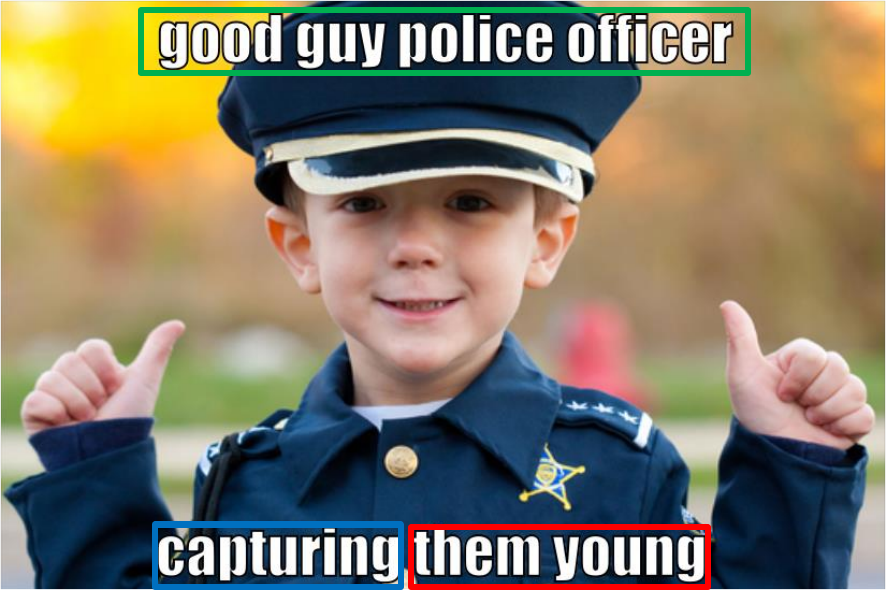}
        \caption{Only presupposed context}
        \label{fig:example2}
    \end{subfigure}
    \hfill
    \begin{subfigure}{0.32\textwidth}
        \includegraphics[width=\linewidth]{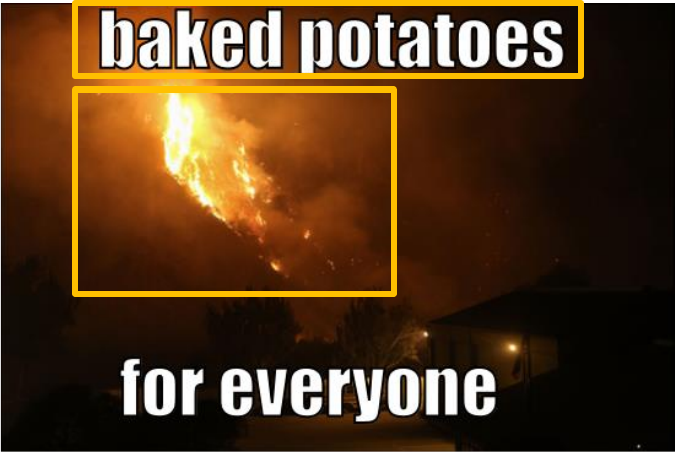}
        \caption{Only false claim}
        \label{fig:example6}
    \end{subfigure} 
    \vspace{-3mm}
    \caption{\scriptsize Examples of how memes express hate through \textbf{presupposed context} and \textbf{false claims}. Figures~\ref{fig:example1} and~\ref{fig:example2} share the text \textit{``good guy police officer, capturing them young.''} \textcolor{visioBlue}{Blue boxes} indicate elements reflecting presupposed context; \textcolor{visioGreen}{green boxes} mark entities framed as ``good''; \textcolor{visioRed}{red boxes} mark those framed as ``bad''; and \textcolor{visioYellow}{yellow boxes} denote text--image referential links. In Figures~\ref{fig:example1} and~\ref{fig:example2}, ``capturing'' conveys an evaluative context in which the white police officer is framed as ``good'', while ``them young'' refers to the Black child as the ``bad'' group. Figure~\ref{fig:example1} is hateful because it also contains false claims: ``capturing'' contradicts the friendly handshake, and the meme misleadingly reinforces the stereotype that Black children are inherently criminal. By contrast, Figure~\ref{fig:example2} lacks explicit false claims, so presupposed context alone is insufficient for hatefulness. Figure~\ref{fig:example6} contains a false claim, describing a volcanic eruption as ``baked potatoes,'' but lacks presupposed context and is better interpreted as dark humor.}
    \vspace{-6mm}
    \label{fig:examples}
\end{figure*}
To combat the spread of hateful memes, Facebook launched \textit{Hateful Meme Challenge}
to promote research on hateful meme classification. Most existing methods rely on pretrained models, especially pretrained language models~\cite{zhou2021multimodal,muennighoff2020vilio,kumar2022hate,burbi2023mapping,cao2023pro}. These methods leverage the unique reasoning capabilities of different models, and optimize hateful meme classification task according to each model's design principles. For instance, PromptHate~\cite{cao2023prompting} employs image-to-text techniques~\cite{mokady2021clipcap,li2023blip} to convert the visual part of memes into captions. This transforms the multimodal task into a pure text classification, enabling the use of pretrained language models~\cite{liu2019roberta}. Since Large Language Models (LLMs) tend to interpret hateful memes as non-hateful, ExplainHM~\cite{lin2024towards} addresses this limitation by generating multi-perspective explanations to improve detection. These methods adapt the detection task to pretrained models' characteristics, achieving strong performance. Despite increasing attention, few studies have explored the fundamental nature of hateful memes—that is, \textit{what renders a meme hateful from psychological and philosophical perspectives?}

In contrast to prior approaches, our goal is to uncover the essence of hateful memes through the lens of \textit{how hate is expressed}, and to design a model informed by these insights. Viewed from this perspective, hateful meme classification faces three key challenges: \textcircled{1} Variation in hate intensity and rhetorical forms (e.g., dehumanizing or an ethnic slur);  \textcircled{2} Limited textual context and intricate interaction between text and image; \textcircled{3} The need for cultural and historical knowledge to recognize hate. To address the above challenges, we ground our modeling in philosophical and psychological theories of hate and hate speech. For \textcircled{1}, we draw on \textit{Marques}' interpretation of hate speech~\cite{marques2023expression}, which posits that all forms of hate share a \textbf{presupposed evaluative context}, i.e., assumptions about who is considered good or bad. For \textcircled{2} and \textcircled{3}, we refer to \textit{Vendrell}, who argues that malicious and ideological hate often relies on vague or false reasoning~\cite{vendrell2021hate}. Accordingly, we identify \textbf{false claims} as a means of expressing hate, manifested in two ways: (i) \textit{Incorrectness}: The text misrepresent the image, creating a mismatch detectable through referential links between textual tokens and visual patches; (ii) \textit{Deliberately misleading}: Embedding stereotypes within text–image interactions to induce biased beliefs, recognition of which requires broader societal knowledge.

To better understand these concepts, consider Figure~\ref{fig:example1} as an example: (1) The presupposed context is often conveyed through verbs, adjectives or actions depicted in the image. For example, the words ``capturing'' and ``good'' reflect an implicit value judgment: the white police officer is framed as the ``good'', while ``them''---the individuals being captured---are implicitly ``bad''. Furthermore, the contrast between ``capturing'' and the friendly atmosphere in the image suggests irony as a means of expressing hate. (2) Multiple referential links exist between the text and the image. The three colored boxes correspond to three types of referential mappings: the red box links ``them young'' and the Black child, the green box connects the ``good guy police officer'' to the white officer, and the blue box links ``capturing'' to the friendly handshake. The mismatch between ``capturing'' and the amicable interaction reveals the incorrectness inherent in the false claim. However, while the presupposed context and the referential links suggest an oppositional relationship between the Black child and the white police officer, this alone is insufficient to determine that the meme is hateful. We also need (3) an understanding of the potential stereotypes underlying this relationship. Referring to the Black community as ``them'' deliberately conveys the impression that the entire group is inherently violent or criminal. When combined with societal tension between Black individuals and White police officers, these elements render the meme interpretable as hateful. Such societal knowledge is critical: if the Black child were replaced with a white child, or the white officer with a Black officer, the oppositional structure might still remain, but its connotation would be unclear, and the perceived level of hate would likely be reduced. It is important to note that \textit{the presence of either a presupposed context or false claims alone does not suffice to indicate a hateful meme.} When only a presupposed context exists without a false claim, the meme may simply express a legitimate viewpoint---for example, Figure~\ref{fig:example2} suggests that police action against young offenders is justified. Conversely, when false claims appear without a presupposed context, the meme often resembles dark humor or fake news, as in Figure~\ref{fig:example6}, which describes a volcanic eruption as handing out ``baked potatoes.''

\noindent \textbf{Present work.} We unpack the problem from the perspective of \textit{how hate is expressed}. We propose that hate in memes is often conveyed through \textbf{presupposed context} and \textbf{false claims}, and we design corresponding modules to model these two aspects: (1) \textbf{Presupposed Context Module (PCM)} encodes intra-modal context and fuses cross-modal contextual information to determine whether a meme conveys an implicit value judgment. (2) \textbf{False Claims Module (FACT)} addresses falsehoods in memes via two sub-modules: \textbf{Social Perception Module (SPM)} introduces external societal knowledge to identify deliberately misleading content. Building on this, we further propose the \textbf{Cross-modal Reference Module (CRM)}, which explicitly constructs the reference relations between image and text to detect semantic incorrectness. Integrating PCM and FACT, we propose \textbf{\textsf{SHIELD}} (\textbf{S}peech \textbf{H}ate
\textbf{I}dentification through
\textbf{E}valuative context and
Falsehoo\textbf{D}), a framework that enhances hateful meme classification by capturing the core features of hate in memes. Our contributions can be summarized as follows: 
\begin{itemize}
    \item We draw on philosophy and psychology to examine the essence of hateful memes via the lens of \textit{the expression of hate}, arguing that a meme is perceived as hateful primarily because it exhibits two expressive characteristics: \textbf{presupposed context} and \textbf{false claims}. (Section~\ref{sec:introduction} and ~\ref{sec:PCFACT})
    
    \item Based on these characteristics, we design the \textbf{PCM} and \textbf{FACT} modules, and integrate them into the \textbf{\textsf{SHIELD}} framework, whose effectiveness is validated against multiple baselines. (Section~\ref{sec:PCFACT}, ~\ref{sec:rq1})
    
    \item We evaluate model specialization and generalization across different hate targets (Appendix~\ref{app:sepcificity}) and assess \textbf{\textsf{SHIELD}} on fake news classification to demonstrate its versatility and generalization (Section~\ref{sec:rq5}).
\end{itemize} 

\section{Related Work.} 

Following Vendrell’s framework of hate~\cite{vendrell2021hate}, we focus on \textit{malicious hate} and \textit{ideological hate}, which are the forms most commonly studied in hate speech detection. Hate speech is generally defined by institutions and platforms as content that conveys negative sentiment toward a target group~\cite{eu2008,un,meta2025,x2023,youtube2025}. In this work, hateful memes are treated as a multimodal form of hate speech, with memes serving as the medium.

\noindent \textbf{Hateful Meme Detection.} 
The \textit{Hateful Memes Challenge}~\cite{kiela2020hateful} established a standard benchmark and showed that unimodal models such as BERT~\cite{devlin2018bert} and Faster R-CNN~\cite{ren2016faster} perform substantially worse than multimodal models such as VisualBERT~\cite{li2019visualbert}, motivating later work on multimodal fusion~\cite{zhang2020hateful,suryawanshi2020multimodal,kumar2022hate}. Since hateful memes often require social and cultural context, many studies fine-tune pretrained multimodal models for the task~\cite{lippe2020multimodal,zhu2020enhance,zhou2021multimodal,velioglu2020detecting,pramanick2021momenta}. Another line of work converts images into captions and feeds caption-text pairs into PLMs~\cite{cao2023prompting,ji2023identifying,mokady2021clipcap,liu2019roberta}. However, caption-based approaches are sensitive to caption quality and may omit critical cues such as race or gender~\cite{cao2023pro}. To address this issue, Pro-Cap uses targeted prompting with BLIP-2~\cite{cao2023pro,li2023blip}, while ISSUES maps images into single-word tokens through textual inversion~\cite{burbi2023mapping,baldrati2023zero}. More recently, RGCL improves discrimination through contrastive learning~\cite{mei2024improving}, ExplainHM introduces LLM-generated explanations~\cite{lin2024towards,touvron2023llama}, and Mod-HATE studies low-resource adaptation with task-specific LoRAs~\cite{cao2024modularized}.

  

\noindent\textbf{MLLMs for image classification.}
MLLMs~\cite{liu2024visual,liu2024llava,li2023blip,bai2023qwen} have recently been used for image classification in two main ways. One line relies on decoded text, such as captions or inferred statements, which are then passed to another language model for prediction~\cite{lin2024towards}. The other uses latent representations directly. Prior work shows that MLLMs retain useful information in their latent space even when it is not effectively decoded in textual outputs~\cite{zhang2024visually}. Motivated by this finding, we leverage latent MLLM features for hateful meme classification.


\section{Problem Statement}
\label{sec:ps}
Given a multimodal meme $\mathcal{M}=\{\mathcal{I}, \mathcal{T}\}$, where $\mathcal{I}$ denotes an image and $\mathcal{T}$ a text sequence, the goal of hateful meme detection is to determine whether $\mathcal{M}$ conveys hate by jointly leveraging visual and textual information.

We approach this task from the perspective of \textit{how hate is expressed}. Drawing on insights from psychology and philosophy, we seek to identify the expressive mechanisms through which a meme $\mathcal{M}$ communicates hate, and to design models that capture features aligned with these mechanisms. However, this expression-oriented view introduces three major challenges:
\vspace{-1mm}
\begin{itemize}
    \item[\textbf{C1}] \textbf{Diversity of hateful expression.} Hate can be conveyed explicitly or implicitly, encompassing varied rhetorical strategies and affective tones, such as vilification or mockery;
    
    \item[\textbf{C2}] \textbf{Dependence on sociocultural knowledge.} Hate speech often reinforces stereotypes about specific groups, requiring background knowledge to distinguish legitimate commentary from distorted or exaggerated attacks; and
    
    \item[\textbf{C3}] \textbf{Complexity of meme representation.} Meme text is typically brief and relies on the accompanying image for interpretation, demanding effective fusion of multimodal cues. 
\end{itemize}



\section{\textsf{SHIELD}: Presupposed Context \& False Claim Hateful Meme Detection Framework.}
\label{sec:PCFACT}

\begin{figure*}[thbp]
  \centering
  \includegraphics[width=\textwidth]{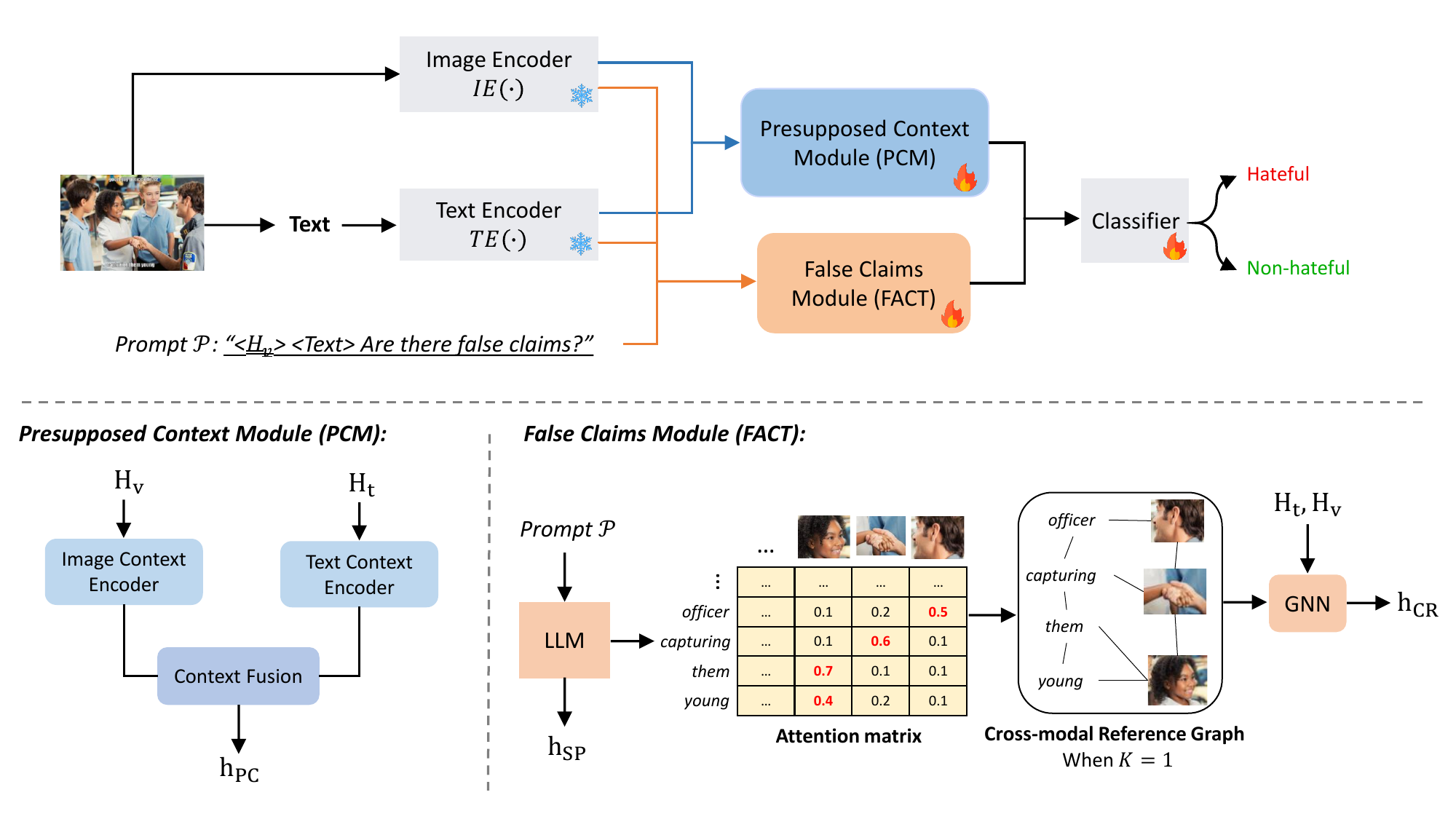} 
  
  \vspace{-7mm}
  \caption{\scriptsize \textsf{SHIELD} framework. Given a meme image and its text, we first obtain patch and token embeddings $H_v$ and $H_t$ from $IE(\cdot)$ and $TE(\cdot)$. (1) In PCM, image and text context encoders model intra-modal interactions to extract modality-specific context features, which are fused to produce the meme context embedding $h_{PC}$. (2) In FACT, $H_v$, $H_t$, and a prompt $\mathcal{P}$ are fed into the LLM, which outputs the last hidden state $h_{SP}$ and an attention matrix. The attention matrix is then used to construct a cross-modal reference graph, which is processed by a GNN to obtain the reference graph embedding $h_{CR}$. Flame icons denote trainable modules, and snowflake icons denote frozen modules.}
  \vspace{-5mm}
  \label{fig:framework}
\end{figure*}

To tackle these three challenges, we construct \textbf{SHIELD} based on the essential characteristics that distinguish hateful memes. The key distinction between hateful memes and humorous or satirical memes lies in intent: hateful memes aim to amplify and disseminate negative stereotypes or prejudices against a specific group, contributing to their social exclusion. Creators of hateful memes employ various strategies, which, from philosophical and psychological perspectives, these expressions in hateful memes typically manifest as \textbf{presupposed context} and \textbf{false claims}:

\noindent\textbf{Hateful meme contains presupposed context (PC).} From a philosophical viewpoint, Marques~\cite{marques2023expression} argues that hate speech is not merely the expression of negative emotions such as anger or contempt, but rather a normative emotional stance (such as the claim that a certain group ought to be excluded). By presupposing the appropriateness of such emotions, hate speech injects value judgments into the context. Social psychology provides a complementary explanation, Tajfel’s Social Identity Theory~\cite{mcleod2023social} posits that individuals categorize others into in-groups and out-groups, generally viewing their own in-group more positively while perceiving out-groups as neutral or negative, thereby enhancing self-image and gaining psychological satisfaction. Linguistically, this aligns with the concept of 'othering language'~\cite{alorainy2019enemy}, i.e., using language to construct difference, exclusion, and denigration, exemplified by 'capturing them young' in Figure~\ref{fig:example1}. Therefore, for \textbf{C1}, despite variations in degree and rhetorical form, all types of hate share a common key feature: \textit{the presence of a presupposed evaluative context}. In \textsf{SHIELD}, this is modeled via \textbf{PCM} in Section~\ref{sec:PCM}.

\noindent\textbf{Hateful meme contains false claims (FC).} Presupposed context alone is insufficient to determine hate. For example, Figure~\ref{fig:example2} merely illustrates an act of justice. To legitimize hate, memes often ascribe unfounded accusations or misrepresentations against the targeted group. Vendrell~\cite{vendrell2021hate} classifies hate into different types. In practice, hateful memes mostly fall under malicious and ideological hate, where the focus is indeterminate, and reasoning is driven by ideology or bias rather than facts. Consequently, the logic behind hateful messages is often vague or false. Therefore, false claims often function as a strategy of attack. False claims in memes primarily manifest in two ways: (1) Incorrectness: incorrect references between the image and textual content. (2) Deliberately misleading: false guidance about target groups, requiring societal knowledge to recognize (\textbf{C2}). This is implemented via the \textbf{SPM} in Section~\ref{sec:SPM}. To capture the incorrectness(\textbf{C3}), we introduce \textbf{CRM} in Section~\ref{sec:CRM}, which explicitly models the associations between image and text.

\subsection{Presupposed Context Module.}
\label{sec:PCM}
A meme’s context comprises both image and text, and capturing the presupposed context requires fusing information across modalities. Prior work has shown that PLMs implicitly encode syntactic structures~\cite{hewitt2019structural}, which can capture the presupposed context in certain degree~\cite{burnap2016us,alorainy2019enemy}. We first encode the image $\mathcal{I}$ and text $\mathcal{T}$ using an image encoder $IE(\cdot)$ and a text encoder $TE(\cdot)$, producing patch and token embeddings $H_v$ and $H_t$. Mean pooling over these embeddings yields the final image embedding $h_v$ and text embedding $h_t$:
\begin{equation}
    H_v = IE(\mathcal{I}), ~~H_t = TE(\mathcal{T})
    \label{eq:encoders}
\end{equation}
\vspace{-11mm}

\begin{equation}
    h_v = \text{Mean}(H_v),~~h_t = \text{Mean}(H_t),
\end{equation}
where $H_v \in \mathbb{R}^{n_v \times d_v}$, $H_t \in \mathbb{R}^{n_t \times d_t}$, $h_v \in \mathbb{R}^{d_v}$, $h_t \in \mathbb{R}^{d_t}$. $n_v$ and $n_t$ denote the number of patches and tokens; $d_v$ and $d_t$ are dimensions of patch embedding and token embedding, respectively. 

\begin{figure}[t]
    \centering
    \begin{subfigure}{0.25\textwidth}
        \includegraphics[width=\linewidth]{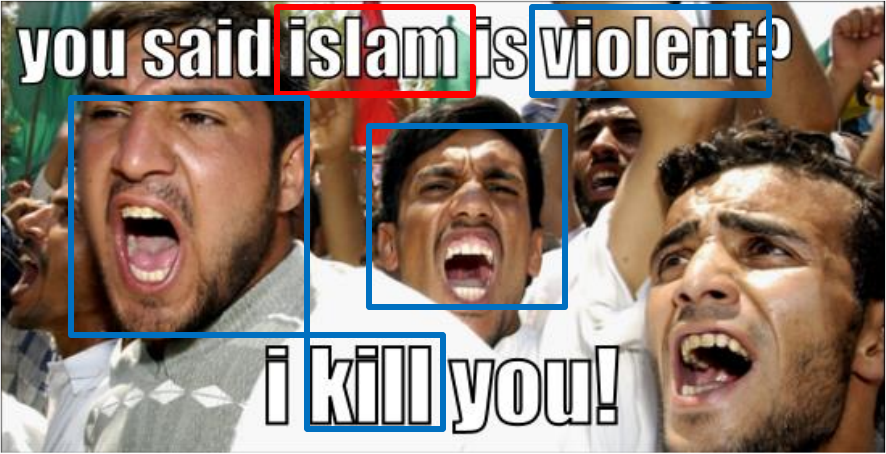}
        \caption{Consistent expression}
        \label{fig:example7}
    \end{subfigure}
    \hfill
    \begin{subfigure}{0.22\textwidth}
        \includegraphics[width=\linewidth]{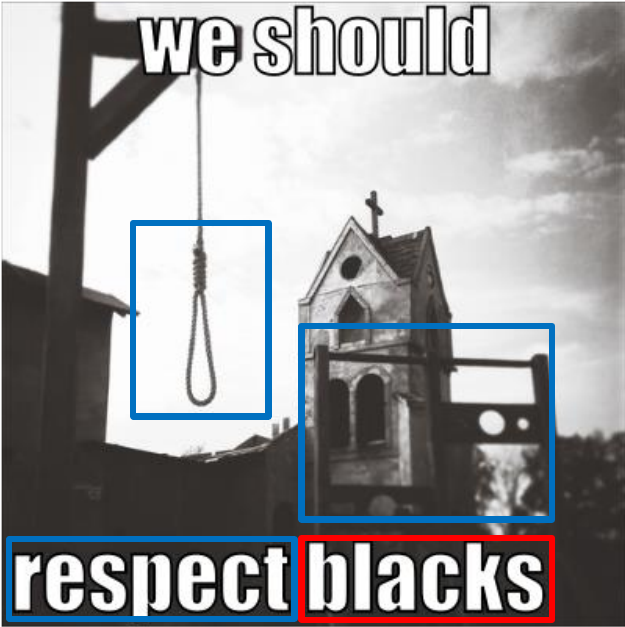}
        \caption{Inconsistent expression}
        \label{fig:example8}
    \end{subfigure} 
    \caption{\small Examples of two context types of hateful memes. In Figure~\ref{fig:example7}, Islam is framed as the bad group. The words ``violent'' and ``kill'' are reinforced by the image, which shows a defiant and confrontational crowd, thereby strengthening the negative portrayal through visual alignment. In contrast, Figure~\ref{fig:example8} uses the positive term ``respect'' to describe an execution scene with nooses and guillotines. This contrast between the text and image creates ironic humor.} 
    \vspace{-5mm}
\end{figure}

To obtain the meme's context embedding, we first apply separate context encoders for the image and text to capture intra-modal contextual information. The context embeddings from both modalities are then fused to generate the final context representation of the meme. For two context encoders, we use two separate linear layers to model intra-modal interactions for each modality. The fusion strategy is guided by the contextual characteristics of hateful memes. Based on how hate is expressed, we categorize meme contexts into two types: (1) Consistent expression, the image and text convey aligned emotions, often delivering hate more explicitly (e.g., Figure~\ref{fig:example7}, where a defiant image amplifies the tone of the text); (2) Inconsistent expression, the image and text convey contrasting emotions, often expressing hate more implicitly (e.g., Figure~\ref{fig:example8}, where the text says ``respect'' but the image depicts a dark execution scene, forming a stark emotional contrast). Therefore, we adopt the Hadamard product of linear transforms of image and text embeddings for fusion:
\begin{equation}
    h_{PC} = (W_1 h_v + b_1) \odot (W_2h_t + b_2),
\end{equation}
where $W_1 \in \mathbb{R}^{d \times d_v}$, $W_2 \in \mathbb{R}^{d \times d_t}$, $b_1$, $b_2 \in \mathbb{R}^{d}$, and $h_{PC} \in \mathbb{R}^{d}$. This approach offers two key advantages: (1) it aligns information across corresponding dimensions of the two modalities, and (2) it enables both intra-modal feature interaction and cross-modal information integration. The Hadamard product is particularly effective in capturing the two context types discussed above. When the modalities express consistent semantics, their aligned dimensions are amplified via multiplication. When they diverge, the resulting negative values highlight the inconsistency. Both effects are preserved and emphasized through this fusion method.

\subsection{False Claim Module.}
In addition to presupposed context, another common feature of hateful memes is the presence of false claims. These can be identified from two perspectives:  (1) Whether the meme deliberately promotes negative stereotypes based on social knowledge, and (2) whether the reference between image and text is accurately aligned. To address these aspects, we successively propose  the Social Perception Module and the Cross-modal Reference Module.

\subsubsection{Social Perception Module (SPM).}
\label{sec:SPM}
To enable the model to detect factual distortions or bias amplification in memes, incorporating external knowledge is crucial. LLMs trained on diverse social media data possess a certain level of social awareness that can help address this challenge. However, leveraging LLMs introduces two key issues: (1) LLMs inherently exhibit cognitive bias and often under-identify hateful content, which may impair classification performance if used directly~\cite{lin2024towards}. (2) Critical knowledge is often retained in the latent space and is difficult to decode effectively~\cite{zhang2024visually}. 

To overcome these challenges, we first design a prompt template $\mathcal{P}$ as input to the LLM: ``<$H_v$> <$\mathcal{T}$> Are there false claims?'' where angle brackets indicate placeholders. Specifically, <$H_v$> denotes the placeholder for the patch embeddings. We encode ``<$\mathcal{T}$> Are there false claims'' using the LLM's embedding layer and concatenate it with $H_v$ before feeding the combined representation into the LLM. To mitigate the inherent bias of the LLM toward regarding memes as non-hateful, we fine-tune it to adapt it to the hateful meme detection task. To fully exploit the knowledge encoded in the LLM, we extract the last hidden state $h_{SP}$ for downstream classification:
\begin{equation}
    h_{SP}, A = LLM_{\mathrm{fine-tuned}}(\mathcal{P})
\end{equation}
where $h_{SP} \in \mathbb{R}^{d_{SP}}$; $A \in \mathbb{R}^{n \times n}$ is the attention matrix with $n$ tokens in total. $LLM_{\mathrm{fine-tuned}}$ denotes the LLM fine-tuned during training, without restricting the specific fine-tuning method. In our implementation, we adopt LoRA~\cite{hu2021lora} for fine-tuning.

\subsubsection{Cross-modal Reference Module (CRM).}
\label{sec:CRM}
To identify incorrect referential relations between image and text in memes, such as the example in Figure~\ref{fig:example1}, where the word ``capturing'' incorrectly refers to a friendly handshake highlighted in the blue boxes. We propose the \textit{cross-modal reference graph} to explicitly model the referential links between image patches and text tokens, thereby supporting the hateful meme detection task.

\noindent\textbf{Cross-modal Reference Graph.} The cross-modal reference graph consists of two types of nodes: text tokens $v_{t}$ and image patches $v_{p}$, with three types of undirected edges capturing their relationships: (1) Token-Token edges $E_{tt}$ follow the natural token order in the text. (2) Patch-Patch edges $E_{pp}$ connect adjacent image patches to preserve spatial structure. (3) Token-Patch edges $E_{tp}$ represent interactions between text tokens and image patches.

\noindent\textbf{Constructing token-patch edges $E_{tp}$.} In SPM, both image and text are included in the prompt $\mathcal{P}$, and the resulting output $h_{SP}$ is used for hateful meme detection. During fine-tuning, the LLM learns the associations between image and text, which are reflected in its attention scores. We therefore leverage the attention matrix $A \in \mathbb{R}^{n \times n}$ to construct referential relationships between two modalities, where $n$ denotes the number of tokens. The index range for text tokens in $\mathcal{T}$ is $[i_1, i_2]$, and for image patch tokens is $[j_1, j_2]$. For each text token $i$, its patch token neighbors are defined as:
\begin{equation}
    S_i = \mathrm{Top-K~Indices}(A_{i,j_1:j_2}, K).
    \label{eq:topk}
\end{equation}
Where $S_i \in [j_1, j_2]$ denotes the indices of the top-K image patch token with the highest attention scores for token $i$. The hyperparameter $K$ controls the number of token-patch connections per text token. We then construct the edge set $E_{tp}$ as:
\begin{equation}
    E_{tp} = \bigcup_{i=i_1}^{i_2} \left\{(i, j) \;\middle|\; j \in S_i \right\}.
\end{equation}
$E_{tp}$ connects each text token $i$ to its most relevant patch tokens $j \in S_i$. The initial features of $v_t$ and $v_p$ are $H_t$ and $H_v$ computed from Eq~\ref{eq:encoders}.

\noindent\textbf{Harnessing Cross-modal Reference Graph.} The constructed cross-modal reference graph $G=(V=\{v_t, v_p\}, E=\{E_{tt}, E_{pp}, E_{tp}\})$ is then fed into a Graph Neural Network (GNN), which gives node embeddings. These node embeddings are aggregated via mean pooling to produce the graph embedding $h_{CR}$:
\begin{equation}
    h_{CR} = \mathrm{\text{Mean}}(\mathrm{GNN}(G, H_v, H_t)),
\end{equation}
where $h_{CR} \in \mathbb{R}^{d_{CR}}$, and $d_{CR}$ denotes the node embedding dimension. A theoretical analysis about reference graph is provided in Appendix~\ref{app:analysis_graph}.

\subsection{Classifier.}
To apply the above three modules to the hateful meme classification task, we combine their embeddings by directly concatenating them to form the final representation $h$:
\begin{equation}
    h = [h_{PC}, h_{SP}, h_{CR}],
    \label{eq:h}
\end{equation}
where $h \in \mathbb{R}^{d+d_{SP}+d_{CR}}$. Then, $h$ is fed into a binary classifier to predict whether the meme is hateful.

\section{Experiments.}
In this section, we detail the experiments conducted to evaluate the performance and capabilities of  \textsf{SHIELD}. We aim to address the following research questions:
\begin{description}
    \item[\textbf{RQ1:}] How does \textsf{SHIELD} perform on hateful meme classification compared with existing models?
    
    \item[\textbf{RQ2:}] How effective are the two proposed modules?
    \item[\textbf{RQ3:}] How does \textsf{SHIELD} perform on other multimodal social media tasks such as fake news detection? 
\end{description} 
Besides, we further evaluate model specificity and generalization across different attack targets in Appendix~\ref{app:sepcificity} and analyze the sensitivity of parameter $K$ in Eq.~\ref{eq:topk} in Appendix~\ref{sec:parameter}.

\begin{table}[H]
    \centering
    \caption{Summary of hateful meme datasets.}
    \label{tab:data1}
    \small
    \begin{tabular}{lcccc}
         \hline
          & & FHM & Harm-C & Harm-P \\ \hline
          \multirow{2}{*}{Train} & \#Harmful  
          & 3019 & 1064   & 621    \\
          & \#Benign & 5481 & 1949   & 616    \\ \hline
          \multirow{2}{*}{Validation} & \#Harmful  & 247  & 61     & 79     \\
                       & \#Benign & 253  & 116    & 80     \\ \hline
\multirow{2}{*}{Test}  & \#Harmful  & 1380 & 124    & 156    \\
                       & \#Benign & 1760 & 230    & 163  \\ \hline
    \end{tabular} 
    \vspace{-3mm}
    
\end{table}

\begin{table*}[t]

\caption{Hateful meme classification results. Best results are in bold; runner-up results are underlined.}
\label{tab:exp1}
\centering
\footnotesize

\begin{tabular}{lccccccccc}
\hline
Dataset                               & \multicolumn{3}{c}{FHM}                                                              & \multicolumn{3}{c}{Harm-C}                                                           & \multicolumn{3}{c}{Harm-P}                                                     \\ \hline
\multicolumn{1}{l|}{Model}           & AUC                 & Acc.            & \multicolumn{1}{c|}{F1}            & AUC                 & Acc.            & \multicolumn{1}{c|}{F1}            & AUC                                & Acc.           & F1            \\ \hline \hline
\multicolumn{1}{l|}{CLIP Text-Only}      & 63.01  & 63.17          &  \multicolumn{1}{c|}{56.45}          & 86.70        & 78.99         & \multicolumn{1}{c|}{76.96}          & 67.87                       & 61.32         & 60.70        \\
\multicolumn{1}{l|}{CLIP Image-Only}      & 79.57  & 75.10          & \multicolumn{1}{c|}{70.75}          & 92.36          & 84.19          & \multicolumn{1}{c|}{82.87}          & 68.53                         & 62.14           & 61.68          \\ \hdashline
\multicolumn{1}{l|}{PromptHate}      & \underline{81.76} & 75.34          & \multicolumn{1}{c|}{73.66}          & 87.50          & 81.30          & \multicolumn{1}{c|}{79.76}          & 52.67                         & 49.35          & 33.04          \\
\multicolumn{1}{l|}{HateCLIPper}     & 80.82          & 76.46          & \multicolumn{1}{c|}{72.75}          & 91.61          & 83.93 & \multicolumn{1}{c|}{82.67} & 70.62                         & 63.89          & 63.72          \\
\multicolumn{1}{l|}{ISSUES}          & 81.66          & \underline{77.33} & \multicolumn{1}{c|}{\underline{74.06}} & \underline{92.39} & 81.98          & \multicolumn{1}{c|}{80.85}          & \underline{71.06}                         & \underline{64.08}          & \underline{63.79}          \\
\multicolumn{1}{l|}{ExplainHM}       & 79.03          & 72.70          & \multicolumn{1}{c|}{70.37}          & 89.31          & 82.73          & \multicolumn{1}{c|}{81.82}          & 62.97                         & 58.97          & 58.62          \\
\multicolumn{1}{l|}{Zero-Shot} & 50.05           & 39.62           & \multicolumn{1}{c|}{28.72}           & 50.44           & 35.60            & \multicolumn{1}{c|}{26.92}           & 50.00                            & 48.91           & 32.85           \\
\multicolumn{1}{l|}{Fine-tuned MLP}  & 70.13          & 63.63          & \multicolumn{1}{c|}{62.98}          & 91.43          &\underline{86.39}          & \multicolumn{1}{c|}{\underline{85.45}}          & 63.63 & 57.50          & 57.42          \\ \hdashline
\multicolumn{1}{l|}{\textsf{SHIELD}}            & \textbf{87.51} & \textbf{80.12} & \multicolumn{1}{c|}{\textbf{79.14}} & \textbf{93.81} & \textbf{89.36} &\multicolumn{1}{c|}{\textbf{88.74}} & \textbf{72.58}                & \textbf{67.60} & \textbf{67.42} 
\\ \hline 
\end{tabular}
\vspace{-3mm}
\end{table*}

\begin{table*}[t]
\caption{Ablation studies by using different modules.} 
\vspace{-2mm}
\label{tab:exp2}
\centering
\footnotesize
\begin{tabular}{lccc|ccc|ccc}
\hline
Dataset                               & \multicolumn{3}{c}{FHM}                                                              & \multicolumn{3}{c}{Harm-C}                                                           & \multicolumn{3}{c}{Harm-P}                                                     \\ \hline 
\multicolumn{1}{l|}{Model}           & AUC                 & Accuracy            & Macro-F1            & AUC                 & Accuracy            & Macro-F1            & AUC                                & Accuracy            & Macro-F1            \\ \hline \hline
\multicolumn{1}{l|}{SPM} & 86.68          & 79.25          & 78.25          & {93.32} & \underline{88.78} & \underline{88.07} & 60.61          & 56.75          & 56.72           \\
\multicolumn{1}{l|}{SPM+PCM}  & \underline{87.26} & \underline{79.41} & \underline{78.33} & \underline{93.74}          & 88.33          & 87.63          & \underline{71.74} & \underline{66.97} & \underline{66.85}          \\ 
\multicolumn{1}{l|}{\textsf{SHIELD}}            & \textbf{87.51} & \textbf{80.12} & \textbf{79.14} & \textbf{93.81} & \textbf{89.36} & \textbf{88.74} & \textbf{72.58} & \textbf{67.60} & \textbf{67.42} 
\\ \hline 

\hline 
\end{tabular} 
\vspace{-5mm}
\end{table*}



\subsection{Experimental Setup.}
\subsubsection{Datasets.}\label{sec:datasets}We evaluate models on three publicly available meme datasets: (1) FHM~\cite{kiela2020hateful}, released by Facebook as part of a multimodal hateful meme classification challenge. It is worth noting that many existing studies test models using different splits of the data (e.g., omitting the validation set~\cite{cao2023prompting,lin2024towards}, or using different test sets~\cite{kumar2022hate,burbi2023mapping}). As labels for the entire dataset are available we use the full set for fair comparison and consistent comparison: the \textit{train} set used for training, the \textit{dev\_seen} set used for validation, and the remaining sets (\textit{dev\_unseen, test\_seen} and \textit{test\_unseen}) used for testing. (2) Harm-C~\cite{pramanick2021detecting}, consisting of COVID-19-related memes. (3) Harm-P~\cite{pramanick2021momenta}, containing US politics-related memes. Both were collected from various web resources, including Reddit, Facebook, and Instagram, but mainly Google Images. The training, validation, and testing sets follow the splits defined in the original paper. However, original Harm-P contains a large number of duplicate memes, which appear in the training, validation, and testing sets, has sometimes been overlooked, leading to inaccurate model evaluation. We remove duplicates by prioritizing samples in the order of train $>$ valid $>$ test, since the latter two contain fewer samples. The cleaned data statistics are summarized in Table~\ref{tab:data1}.

\subsubsection{Baselines.}We compare with several state-of-the-art unimodal and multimodal hateful meme detection models, including CLIP, PromptHate~\cite{cao2023prompting}, HateCLIPper~\cite{kumar2022hate}, ISSUES~\cite{burbi2023mapping}, ExplainHM~\cite{lin2024towards}, and both zero-shot and fine-tuned MLP versions of Llama~\cite{touvron2023llama}. Detailed descriptions are provided in Appendix~\ref{sec:app_baselines}.

\subsubsection{Evaluation.}
We adopt the most commonly used evaluation metrics for the hateful meme classification task~\cite{kiela2020hateful,cao2023prompting,kumar2022hate,burbi2023mapping,lin2024towards}: Area Under the Receiver Operating Characteristic Curve (AUC), Accuracy, and Macro-F1 score. To obtain more reliable results, all models were trained and evaluated multiple times, and we report the mean performance across five runs. In all tables, the best and runner-up results are in bold and underlined, respectively. Details about experimental settings as well as alignment across modalities can be seen in Appendix~\ref{sec:app_implementation}.


\subsection{RQ1: Hateful Meme Classification.}
\label{sec:rq1}
Table~\ref{tab:exp1} compares \textsf{SHIELD} with baselines on FHM, Harm-C, and Harm-P. First, multimodal models consistently outperform unimodal ones, showing that visual information is critical for hateful meme detection. CLIP Image-Only also surpasses CLIP Text-Only by a large margin, likely because meme text is embedded in the image and CLIP has implicit OCR ability~\cite{yu2023turning}. Second, among the baselines, ISSUES and HateCLIPper achieve the strongest results, highlighting the importance of stronger multimodal interaction and external knowledge. PromptHate and ExplainHM perform worse in our setting because we use a validation set for checkpoint selection instead of selecting directly on the test set. Third, Zero-Shot performs poorly, while Fine-tuned MLP yields limited improvement, showing that pretrained LLMs alone are insufficient and require fine-tuning for this task~\cite{lin2024towards,ji2023survey}. Finally, \textsf{SHIELD} achieves the best AUC, Accuracy, and F1 on all three datasets, surpassing the best baseline in AUC by 5.75\%, 1.42\%, and 1.52\% on FHM, Harm-C, and Harm-P, respectively.


\subsection{RQ2: Module Performance Analysis.} We conduct ablation experiments to evaluate module effectiveness: 
\begin{enumerate}
    \item \textbf{SPM}: Using only the SPM, that is, using only $h_{SP}$ in Eq.~\ref{eq:h}.
    \item \textbf{SPM+PCM}: Using both SPM and PCM, that is, incorporating  $h_{SP}$ and $h_{PC}$ in Eq.~\ref{eq:h}.
    \item \textbf{\textsf{SHIELD}}: Using all modules by feeding the complete representation $h$ into the classifier.
\end{enumerate}
As shown in Table~\ref{tab:exp2}, the full \textsf{SHIELD} model consistently outperforms all variants. Performance improves on most metrics as each module is added, with especially clear gains on FHM and Harm-P. The improvements on Harm-C are smaller, likely because harmfulness is defined more broadly than hatefulness. For example, in Figure~\ref{fig:harm_c_example}, the meme has no explicit attack target and may be labeled harmful simply because it refers to the dangerous act of drinking a bottle of disinfectant, implying self-harm. In such cases, the harmful behavior is already conveyed clearly by the text, so additional multimodal extraction yields limited benefit. In contrast, Harm-P mainly targets political figures such as Trump and Obama, making it more similar to hateful content. 

\begin{figure}[!th]
    \centering
    \includegraphics[width=\linewidth]{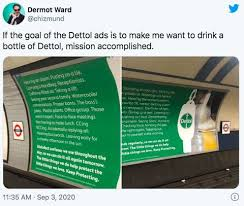}
    \caption{An example of a harmful meme in Harm-C.}
    \label{fig:harm_c_example}
\end{figure}

\begin{table}[h]
\centering
\caption{Performance on the Fake News Classification.}
\label{tab:exp5}
\footnotesize
\begin{tabular}{clcc}
\hline
\centering Dataset &Model &Acc.  & F1\\
\midrule
\multirow{4}{*}{ReCOVery}  
& SAFE     & \underline{93.95}          & \underline{85.77}                 \\
                           & MCAN     & 85.72          & 56.99              \\
                           & CAFE     & 88.70          & 67.68          \\ 
                           & \textsf{SHIELD}     & \textbf{97.61} & \textbf{95.00}  \\
                           \midrule
\multirow{4}{*}{GossipCop} 
& SAFE     & \underline{87.67}          & \underline{78.68}               \\
                           & MCAN     & 81.49          & 51.12           \\
                           & CAFE     & 82.62          & 66.13             \\ 
                           & \textsf{SHIELD}     & \textbf{89.00} & \textbf{81.60}  \\ \bottomrule 

\end{tabular}
\vspace{-5mm}
\end{table}

\subsection{RQ3: Versatility.}
\label{sec:rq5}
Presupposed context and false claims are not exclusive to hateful memes, they also appear in other social media data, such as fake news. If the \textsf{SHIELD} framework indeed captures these features, it should generalize to broader social media tasks. To verify this, we applied \textsf{SHIELD} to the multimodal fake news classification task. For baseline models, we selected several of the most recently released and widely adopted open-source models, including SAFE~\cite{zhou2020similarity}, MCAN~\cite{wu2021multimodal}, and CAFE~\cite{chen2022cross}. For datasets, we used ReCOVery~\cite{zhou2020recovery} and GossipCop~\cite{shu2020fakenewsnet}. More details about results  and implementation can be seen in Appendix~\ref{sec:fake_news}. The results suggest two main insights. (1) Among the baselines, although MCAN and CAFE were proposed after SAFE, neither outperforms it. This may be due to the lack of a validation set in their original training pipelines, which likely caused overfitting and weaker generalization. In contrast, SAFE shows stronger performance across both datasets. (2) Despite text truncation, \textsf{SHIELD} outperforms the baselines on most metrics, demonstrating its applicability to broader multimodal social media tasks and highlighting the prevalence of presupposed context and false claims in social media.


\section{Conclusion.}

We unpack hateful memes by examining how hate is expressed, drawing on insights from philosophy and psychology to identify two key features: a presupposed context and false claims. Based on these two features, we propose two modules, PCM and FACT, which together form the \textsf{SHIELD} framework. PCM captures the meme's context by modeling intra-modal interactions and fusing contextual information across modalities. FACT enhances the model's understanding of hateful content by introducing external knowledge through SPM, and explicitly constructs image-text referential relations via CRM to detect false claims. Experimental comparisons with multiple baselines demonstrate the effectiveness, generalizability, and versatility of \textsf{SHIELD}.

\bibliographystyle{siamplain}
\bibliography{Content/ref}

\appendix

\section{Details of Evaluation.}
\label{sec:app_evaluation}
In this section, we provide details of baselines and implementation.
\subsection{Baselines}
\label{sec:app_baselines}
The detailed description of the comparison baseline is as follows:
\begin{itemize}
    \item \textbf{CLIP Text-Only:} We use the CLIP text encoder as a unimodal text-only model, without additional visual information. After obtaining the text embedding, it is fed into an MLP to directly predict the binary classification task.
    \item \textbf{CLIP Image-Only:} We only use the CLIP image encoder to extract image embeddings as input for the downstream multilayer perceptron (MLP) classifier.
    \item \textbf{PromptHate~\cite{cao2023prompting}:} Leverages PLMs to reframe multimodal detection as masked token prediction task.
    \item \textbf{HateCLIPper~\cite{kumar2022hate}:} Uses CLIP to extract multimodal features, combined with various feature fusion methods.
    \item \textbf{ISSUES~\cite{burbi2023mapping}:} Relies on Textual Inversion~\cite{gal2022image} to enhance multimodal features, while a two-stage training strategy adapts the pre-trained model to hateful meme classification.
    \item \textbf{ExplainHM~\cite{lin2024towards}:} An explainable harmful meme detection approach where LLMs generate contradictory rationales, and a tunable language model judges meme harmfulness based on these rationales. 
    \item \textbf{Zero-Shot}: We use the same $IE(\cdot)$ and LLM to directly decode without training. Given the prompt \textit{``<$\mathcal{M}$> This meme is hateful or not?''}, the prediction depends on whether the response includes ``hateful'' or ``benign.''
    \item \textbf{Fine-tuned MLP:} The SPM variant without a fine-tuned LLM. The obtained $h_{SP}$ is used as the input to a downstream MLP classifier, with only the classifier being trained.
\end{itemize}

\subsection{Implementation Details.}
\label{sec:app_implementation}
We use the GCN~\cite{kipf2016semi} as the GNN model in CRM. The LLM backbone is \textit{meta-llama/Llama-2-7b-chat-hf}; the image encoder $IE(\cdot)$ is \textit{openai/clip-vit-large-patch14-336}, combined with the pre-trained LLaVA projection layer; and the text encoder $TE(\cdot)$ uses the same LLM with pre-trained LoRA~\cite{behnamghader2024llm2vec}; The downstream classifier is an MLP. The hyperparameter $K$ in Eq.~\ref{eq:topk} is tuned based on performance, as detailed in the Parameter Sensitivity Analysis (Appendix~\ref{sec:parameter}). Models are evaluated on the validation set after each epoch, and the best one on the validation set is used for testing. The reported results are based on this test evaluation. All experiments are conducted on 10 Quadro RTX 6000 GPUs.

\section{Parameter Sensitivity Analysis.}
\label{sec:parameter}
In this section, we examine the sensitivity of the hyperparameter $K$ in CRM, which determines the number of patch nodes connected to each token node. All experiments share the same configuration except for the value of $K$, which we vary among \{1, 4, 8, 16\} to assess its influence on model performance across the three datasets.

As shown in Table~\ref{tab:exp3}, the model exhibits noticeable sensitivity to $K$. When $K$ is larger, each token may connect to more irrelevant patches, introducing noise that degrades performance. Conversely, when $K$ is small, fewer patches are connected to each token, which may lead to incomplete token-patch associations, making it difficult for the model to capture their relationships. Additionally, the optimal choice of $K$ varies across different datasets, depending on factors such as meme complexity and image quality. Overall, performance consistently peaks when $K$ lies between 4 and 8, striking a balance between coverage and noise suppression.

\begin{table*}[thbp]
\caption{\small Sensitivity of model performance on the parameter $K$}
\label{tab:exp3}
\centering
\begin{tabular}{lccc|ccc|ccc}
\hline
Dataset                               & \multicolumn{3}{c}{FHM}                                                              & \multicolumn{3}{c}{Harm-C}                                                           & \multicolumn{3}{c}{Harm-P}                                                     \\ \hline 
\multicolumn{1}{c|}{K}           & AUC                 & Accuracy            & Macro-F1            & AUC                 & Accuracy            & Macro-F1            & AUC                                & Accuracy            & Macro-F1            \\ \hline \hline
\multicolumn{1}{c|}{1} & \underline{87.64} & 80.09          & 78.95          & 93.42          & 88.33          & 87.30          & \underline{70.36} & 64.06          & 63.43           \\
\multicolumn{1}{c|}{4}  & 87.51          & \underline{80.12} & \underline{79.14} & \textbf{93.95} & \textbf{89.61} & \textbf{88.89} & \textbf{72.58} & \textbf{67.60} & \textbf{67.42}          \\ 
\multicolumn{1}{c|}{8}            & \textbf{87.69} & \textbf{80.39} & \textbf{79.28} & \underline{93.60} & 87.68          & 86.28          & 69.83          & \underline{66.16} & \underline{63.68} \\
\multicolumn{1}{c|}{16}            & 83.35          & 76.06          & 74.28          & 93.45          & \underline{89.37} & \underline{88.67} & 68.80          & 64.06          & 63.34
\\ \hline 
\end{tabular} 
\end{table*}

\section{Specificity and Generalization.}
\label{app:sepcificity}
In real-world hateful memes detection, the targets and topics of attack vary widely. Thus, two key aspects must be considered: (1) Specificity---whether a model overfits to certain hate targets; and (2) Generalization---whether it captures invariant properties of hate, such as presupposed context and false claims. A model capable of identifying such invariants should generalize across different domains of hate. To evaluate these points, we use Harm-C (focused on COVID-19) and Harm-P (focused on politics). We combine them into a unified dataset, Harm-C + Harm-P, train models on this mixture, and test them separately on the test splits of Harm-C, Harm-P, and the combined dataset.

As shown in Table~\ref{tab:exp4}, several findings emerge: (1) Compared with Table~\ref{tab:exp1}, training on a more heterogeneous dataset, i.e., one that includes a broader range of hate targets, generally reduces performance on individual domains, suggesting a trade-off between specialization and generalization. CLIP-based models such as ISSUES and HateCLIPper experience smaller performance drops on Harm-P, indicating a stronger specialization in detecting political hate. (2) \textsf{SHIELD} performs competitively with ISSUES on both datasets, achieving a 0.85 higher AUC on Harm-C but 0.97 lower on Harm-P. However, on the combined dataset, \textsf{SHIELD} surpasses ISSUES by 2.12 AUC, demonstrating that \textsf{SHIELD} better captures transferable hate-related features and maintains higher discriminative ranking across diverse targets. Overall, these findings suggest that \textsf{SHIELD} achieves a more balanced and robust trade-off between specificity and generalization.

\begin{table*}[t]
\caption{\small Model specificity and generalization across hate targets. Harm-C + Harm-P denotes the combined dataset. Models are trained on the combined dataset and evaluated on the test sets of Harm-C, Harm-P, and Harm-C + Harm-P.}
\label{tab:exp4}
\centering
\footnotesize
\begin{tabular}{lccc|ccc|ccc}
\hline
Dataset  & \multicolumn{3}{c}{Harm-C}                                                           & \multicolumn{3}{c}{Harm-P}        & \multicolumn{3}{c}{Harm-C + Harm-P}                                             \\ \hline
\multicolumn{1}{l|}{Model}                      & AUC                 & Acc.           &F1            & AUC                                & Acc.            &F1    & AUC                                & Acc.            &F1          \\ \hline \hline
\multicolumn{1}{l|}{PromptHate}   & 86.01          & 78.11          & 73.85 & 61.24          & 58.08          & 56.70 & 75.93 & 67.81 & 65.36                                   \\
\multicolumn{1}{l|}{HateCLIPper}     & 89.56  & 82.04          & 80.07          & \underline{68.13}           & \underline{60.32}          & 57.25 & 79.44 &\underline{71.74} & 69.36        \\
\multicolumn{1}{l|}{ISSUES}          & \underline{91.01}          & \underline{82.94}           & \underline{81.08}          & \textbf{68.55}          & 58.56          & 54.10 & \underline{80.37} & 71.39 & 68.43         \\
\multicolumn{1}{l|}{ExplainHM}       & 87.51          & 80.34 & 79.40 & 62.34          & 59.22          & \underline{58.25}  & 77.80 & 70.27 & \underline{70.02}        \\ \hdashline
\multicolumn{1}{l|}{\textsf{SHIELD}}  & \textbf{91.86}          & \textbf{85.32}          & \textbf{84.54}          & 67.58 & \textbf{61.74} & \textbf{60.84} & \textbf{82.48} & \textbf{73.96} & \textbf{73.46}          
\\  \hline  
\end{tabular}
\vspace{-4mm}
\end{table*} 

\begin{table*}[ht]
\centering

\caption{Performance on the Fake News Classification Task.}
\label{tab:exp5_detail}
\begin{tabular}{clcccccccc}
\hline
\multirow{2}{*}{\centering Dataset} &\multirow{2}{*}{Model} &\multirow{2}{*}{Accuracy}  & \multirow{2}{*}{Macro-F1} & \multicolumn{3}{c}{Fake News} &\multicolumn{3}{c}{Real News} \\
\cmidrule{5-10}
& & & & Precision & Recall & F1-score & Precision & Recall & F1-score \\
\midrule
\multirow{4}{*}{ReCOVery}  
& SAFE     & \underline{93.95}          & \underline{85.77}          & 86.24          &\underline{66.67}          & \underline{74.97}          & \underline{94.91}          & 98.29 & \underline{96.56}          \\
                           & MCAN     & 85.72          & 56.99          & \textbf{100}   & 11.67          & 20.52         & 87.71          & \textbf{100}            & 93.45          \\
                           & CAFE     & 88.70          & 67.68          & 75.80         & 30.84         & 41.63         & 89.96          & 97.89          & 93.73          \\ \cdashline{2-10}
                           & \textsf{SHIELD}     & \textbf{97.61} & \textbf{95.00} & \underline{91.39} & \textbf{91.67} & \textbf{91.38} & \textbf{98.68} & \underline{98.55} & \textbf{98.61} \\
                           \midrule
\multirow{4}{*}{GossipCop} 
& SAFE     & \underline{87.67}          & \underline{78.68}          & 74.10          & \underline{57.65}          & \underline{64.82}          & \underline{90.15}          & \underline{95.04} & \underline{92.53}          \\
                           & MCAN     & 81.49          & 51.12          & \textbf{91.33} & 6.79           & 12.60          & 81.36          & \textbf{99.82} & 89.65          \\
                           & CAFE     & 82.62          & 66.13          & 61.04          & 32.59          & 42.49          & 85.16          & 94.90          & 89.76          \\ \cdashline{2-10}
                           & \textsf{SHIELD}     & \textbf{89.00} & \textbf{81.60} & \underline{76.01} & \textbf{65.03} & \textbf{69.94} & \textbf{91.72} & 94.88          & \textbf{93.26} \\ \bottomrule 

\end{tabular}
\end{table*}

\section{Details of Fake News Classification.}
\label{sec:fake_news}
In this section, we provide details of fake news dataset and results.
\subsection{Dataset}
The description of the fake news datasets are as follows:
\begin{itemize}
    \item ReCOVery~\cite{zhou2020recovery} is a collection of COVID-related fake news. The authors first identify reliable and unreliable news sources, and then crawl COVID-19 news articles from these sites to construct the dataset.
    \item GossipCop~\cite{shu2020fakenewsnet} is a fact-checking website for entertainment news aggregated from multiple media outlets. It provides rating scores on a scale of 0–10 to assess the degree of truthfulness, where lower scores indicate more deceptive content. In constructing the dataset, the authors selected fake news items with scores below 5 and paired them with real news sampled from E! Online, a reputable and verified entertainment source.
\end{itemize}
The summary for datasets is shown in Table~\ref{tab:data2}

\begin{table}[htbp]
    \centering
    \caption{Summary of Fake News datasets.}
    \label{tab:data2}
    \begin{tabular}{lccc}
         \hline
          & & ReCOVery & GossipCop \\ \hline
          \multirow{2}{*}{Train} & \#Real & 867                     & 6311                       \\
                       & \#Fake & 158                     & 1627                       \\ \hline
\multirow{2}{*}{Valid} & \#Real & 126                     & 905                        \\
                       & \#Fake & 20                      & 229                        \\ \hline
\multirow{2}{*}{Test}  & \#Real & 249                     & 1821                       \\
                       & \#Fake & 43                      & 447  \\ \hline 
    \end{tabular} 
\end{table}

\subsection{Implementation Details.}
Here, there are three key points to note: (1) MCAN and CAFE did not include a validation set in their original paper; their reported results were the best performance on the test set after each training epoch. In this paper, to ensure a fair and effective comparison of the model, we introduced a validation set. The best-performing model on the validation set was then used for testing. (2) When using \textsf{SHIELD}, we truncated the text in the ReCOVery and GossipCop datasets to the first 150 and 50 words, respectively, to fit the local GPU memory capacity. (3) These datasets only provide image links, many of which have become invalid over time. As a result, the dataset statistics differ from those reported in previous works~\cite{zhou2023multimodal,zhou2023linguistic}.

\section{Analysis of the Reference Graph.}
\label{app:analysis_graph}
While constructing an explicit reference graph can improve detection performance, its effectiveness may be limited in certain cases. In the false claim detection scenario, the label of a graph may change drastically with only minor node-level modifications. For example, in Figure~\ref{fig:framework}, replacing the word “capturing” with “helping” results in a non-false statement, implying that the corresponding graph embedding should ideally be entirely different. Understanding whether such sensitivity is achievable, and to what extent, is crucial for evaluating the applicability of this module. 

Suppose flipping the sign of a node embedding changes the graph’s meaning. Let $\hat{L} \in \mathbb{R}^{n \times n}$ denote the Laplacian matrix of the reference graph, where $n$ is the number of nodes, and node $v_0$ with feature $h_{v_0}$ is switched to $-h_{v_0}$. 
Assume a $K$-layer GCN with mean-pooling readout and no nonlinear activation for analytical simplicity. The binary classification result is $l = u^{\mathrm{T}} s$, where $u$ is the classifier weight and $s$ is the graph embedding. We claim the following:

\begin{theorem}
The reference graph is discriminative if:
\begin{equation}
    \|\delta\|_2 > \frac{n |l|}{\mathbf{1}^{\mathrm{T}} \hat{L}^{(K)} \mathbf{e}_{v_0} \|P u\|_2},
\end{equation}
where $\mathbf{e}_{v_0}$ is a one-hot vector indicating node $v_0$, $\delta = -2 h_{v_0}$, and $P = \prod_t W^{(t)}$ denotes the product of layer-wise GCN weights $W^{(t)}$.

\end{theorem}

\begin{proof}
The initial feature difference can be expressed as $\Delta H^{(0)} = \mathbf{e}_{v_0} \delta^\mathrm{T}$. 
After $K$ propagation layers, the difference in node embeddings becomes $\Delta H^{(K)} = (\hat{L}^{K} \mathbf{e}_{v_0})(\delta^\mathrm{T} P)$. 
The resulting change in the graph embedding is:
\begin{equation}
    \Delta s = \frac{1}{n} \sum_{i=1}^n \Delta h_i^{(K)} = \frac{(\mathbf{1}^{\mathrm{T}} \hat{L}^{(K)} \mathbf{e}_{v_0})(\delta^\mathrm{T} P)}{n}.
\end{equation}
Accordingly, the change in the classification result is:
\begin{equation}
    \Delta l = u^{\mathrm{T}} \Delta s = \frac{(\mathbf{1}^{\mathrm{T}} \hat{L}^{(K)} \mathbf{e}_{v_0})(\delta^\mathrm{T} P u)}{n}.
\end{equation}
Assume $l > 0$. An ideal model should flip the original prediction, i.e., (1) $\mathrm{sign}(\Delta l) \neq \mathrm{sign}(l)$, and (2) $|\Delta l| > |l|$. 
By the Cauchy–Schwarz inequality, we have:
\begin{equation}
    |l| < |\Delta l| \leq \frac{\mathbf{1}^{\mathrm{T}} \hat{L}^{(K)} \mathbf{e}_{v_0}}{n} \|\delta\|_2 \|P u\|_2.
\end{equation}
Hence, the reference graph can achieve the ideal discriminative property if 
{\[
\|\delta\|_2 > \frac{n |l|}{\mathbf{1}^{\mathrm{T}} \hat{L}^{(K)} \mathbf{e}_{v_0} \|P u\|_2}.
\]}
\end{proof}

\end{document}